\documentclass[twoside]{LNGAI}

\usepackage{LNGAImacro}

\usepackage{graphicx}
\usepackage{times}
\usepackage{soul}
\usepackage{url}

\usepackage{mathrsfs}
\usepackage{array}    
\usepackage{amsmath}
 \usepackage{amssymb}
 \usepackage{algorithm}
\usepackage[switch]{lineno}
\usepackage[table]{xcolor}
\usepackage{multirow}
\usepackage{enumitem}
\usepackage{makecell} 
\usepackage{algpseudocode}  
\usepackage{booktabs}
\usepackage{subcaption}

\newcommand{\comm}[1]{}

\newcommand{\instnot}[1]{  \mathit{#1}     }

\def\lastname{Kar, Lorini and Masquelier}

\begin{document}

\begin{frontmatter}
  \title{Binary Spiking Neural Networks as Causal Models}
  \author{Aditya Kar}
  \address{Institut de Recherche en Informatique de Toulouse (IRIT),\\ Centre de Recherche Cerveau et Cognition (CerCo), CNRS, France}
  \author{Emiliano Lorini}
  \address{Institut de Recherche en Informatique de Toulouse (IRIT), CNRS, France}
  \author{Timothée Masquelier}
  \address{Centre de Recherche Cerveau et Cognition (CerCo), CNRS, France}

  \begin{abstract}
    We provide a causal analysis of Binary Spiking Neural Networks (BSNNs) to explain their behavior. 
    We formally define a BSNN 
    and   represent its  spiking activity
      as a binary causal model.
    Thanks to this causal  representation, 
    we are able to explain the output of the network
    by leveraging  logic-based  methods. 
    In particular,
    we show that we  can successfully 
    use a SAT   
    as well as a SMT solver
    to  compute 
      abductive explanations from this  binary causal model. 
    To illustrate our approach, 
    we trained the BSNN on the standard MNIST
    dataset and applied our SAT-based
    and SMT-based
    methods   to
    finding  abductive  explanations of  the network's classifications
    based on pixel-level features. We also compared the found explanations against SHAP,  a popular 
    method used in the area of explainable
    AI.
    We show that, unlike SHAP,
    our approach guarantees  that a found  explanation  does
    not contain completely irrelevant features. 
  \end{abstract}

  \begin{keyword}
  Logic, Causality, Explanation, Binary Spiking Neural Networks 
  \end{keyword}
 \end{frontmatter}

\section{Introduction}
In recent times, interest in the study of binary artificial neural networks has grown,
where binarization can occur at the level of the connection weights between the neural
units, 
at the level of their activation function, or at both levels. 
In the field of AI, binarized neural networks (BNNs) were recently
proposed by \cite{NIPS2016_d8330f85} and \cite{10.1007/978-3-319-46493-0_32},
while in neuroscience particular attention
has been paid to 
binary spiking neural networks (BSNNs) \cite{bs4nn,Lu_2020}. 
The main difference between BNNs and BSNNs is mainly due
to the presence of temporal dynamics in BSNNs over BNNs and
to the fact that in BSNNs inputs are  given sequentially in discrete time,
while they are instantaneously  presented to  BNNs. 
Binarization obviously comes with a price in terms of the size of the network parameters relative to its learning power: a binary neural network requires a considerably higher number of neural units, compared to its non-binary counterpart, in order to achieve an acceptable level of accuracy in a given classification task after training. 
Nonetheless, this disadvantage is counterbalanced by an advantage in terms of 
logical representability and therefore 
explainability.
Specifically, thanks  to the   Boolean nature
of BNNs and  BSNNs, one can represent 
their 
firing dynamics
as  binary  causal models and, consequently, explain their  behaviors   in an efficient way using logic-based methods.

The present paper is devoted to exploring this trade-off between 
accuracy and explainability 
in the context of BSNNs. 
We focus our analysis on BSNNs instead
of BNNs since,
from
the causal point of view, 
the former
are more general than
the latter
and we prefer
to concentrate on the more
general model first.
To 
fully capture the
causal structure of a BSNN, one has to model
the firing activities  of  its  neural units 
and to represent their causal dependencies over an extended time span. 
 BNNs are less general since the presentation of the input
 is not sequential and, consequently, their  dynamics
 and the resulting causal dependencies between  the neural units do not extend over time. 
We represent  
 the internal mechanism
 of a BSNN through a binary causal
 model and, thanks to this representation,
 we  explain the BSNN's behavior. Different
 notions of explanation exist
in the literature including
abductive \cite{ignatiev2019abduction}, contrastive \cite{DBLP:journals/ker/Miller21},
counterfactual \cite{DBLP:journals/corr/abs-2010-10596}
and alterfactual \cite{ijcai2024p52} explanation.
In the present paper, we rely on abductive
explanation (AXp) because of its simplicity
and its emphasis of minimality
which is a guarantee of non-redundancy. 
For a set of input features to be an abductive explanation 
of a classification by a neural network, 
it has to be \emph{minimally} sufficient to ensure the classification, i.e., 
where minimality means that 
all proper subsets of features are no longer sufficient for the classification. 
Thus, an  AXp is by definition non-redundant. 
The causal component will be essential to our analysis. Having an explicit representation of the BSNN’s causal dependencies will enable us to formally verify that all features included in an explanation are genuinely causally relevant, unlike traditional machine learning explainability methods such as SHAP.

 The paper is structured as follows. 
 After discussing related work,
 we illustrate the BSNN architecture
 as well as 
 the learning task 
 we considered, namely
 MNIST classification,   and the learning algorithm we used
 to train our BSNNs
 on the MNIST dataset. 
Then, we focus on the mathematical
aspects of our framework. First, we introduce the mathematical
model of the BSNN spiking dynamics.
Then, we map it onto a binary causal model that represents
the causal dependencies between the firing activities of the neural
units over time. 
 We then move 
to the explanation of the BSNN behavior. Specifically,
we present an algorithm that combines the binary causal model
with a SAT solver 
to compute abductive explanations of the BSNN classification,
where an abductive explanation is constructed from  pixel-level features at 
a specific time point. 
We also consider 
a variant of the abductive
explanation search
algorithm based on a SMT (Satisfiability Modulo Theory)
encoding of the binary causal model
for the BSNN architectures. 
We present some experimental results
on  computation time for  the SAT-based 
and for the SMT-based
explanation search algorithm. 
Finally, in Section
\ref{sec:shap}
we compare our logic-based approach
relying on    abductive explanation
with SHAP.  

To the best of our knowledge, this is the first attempt i) to map a BSNN onto a binary causal model, and ii) to leverage the resulting Boolean representation of causal dependencies among its neural units to explain its behavior using both SAT and SMT solvers. 

\section{Related Work}\label{sec:relwork}

We organize the discussion
of relevant literature in three parts:
binary neural networks,
causal models, 
and logic-based
explanation
of artificial neural networks.

\paragraph{Binary neural networks}
Binary  Neural Networks (BNNs)
are a class
of artificial neural networks (ANNs) 
that 
have been studied extensively by researchers \cite{QIN2020107281} in the deep learning community, especially by \cite{bengio2013estimating} and \cite{NIPS2016_d8330f85},
who provided a viable way to train these networks using standard back-prop based optimisation methods. BNNs adopt an extreme form of quantization, by resorting to binary weights and binary activation values.  \cite{Tang_Hua_Wang_2017}
have shown 
that with back-prop based methods, it is possible to train these binarized neural networks with reasonable, near full precision accuracy. Moreover, \cite{10.1007/978-3-319-46493-0_32} demonstrated a drastic reduction in computation time and model size with XNOR-Nets, due to the fact that the computationally expensive multiply-accumulate operations in deep learning can be replaced by faster XNOR and pop-count operations when using binarized networks. Hence, for these reasons, BNNs have gained significant popularity in resource-constrained, low-power, and hardware-efficient AI applications. 
Binary Spiking Neural Networks (BSNNs), the subject of the present paper, are the bio-plausible counterpart of BNNs, taking inspiration from the spiking dynamics of biological neurons in the brain. The most useful feature of BSNNs is the way they process input data using spike encodings, where spikes are binary all-or-none pulses occurring at discrete time steps, as opposed to the continuous-valued representations used in conventional ANNs (including BNNs). These spike encodings are particularly convenient, as they allow us to apply our formalism to both the pixel space and the intermediate feature space. BSNNs have been trained using both temporal \cite{bs4nn} and rate coding schemes \cite{Lu_2020}.


\paragraph{Causal models}
Causal models are mathematical objects that have been
extensively studied in AI \cite{Pearl2009}, logic \cite{Halpern2000,HalpernBook2016},
and 
in the field of explainable AI \cite{DBLP:journals/ker/Miller21}, 
given the urgent need to provide formally rigorous
causal explanations of  AI systems. 
A causal
model is a system
of structural equations 
describing the causal dependencies between variables. 
Binary causal models (BCMs)
that we use in the present work are the subclass
of causal
models in which variables are assumed to  be Boolean. 
They were studied in depth  in previous work 
\cite{DBLP:journals/jair/ChocklerH04,DBLP:journals/jair/AleksandrowiczC17,ijcai2023p366,DBLP:conf/ijcai/Lorini24}. 
Given their close connection with propositional logic, they offer the possibility to automate reasoning about causality with the aid of a SAT solver.  
 
 \paragraph{Abductive explanation of artificial neural networks}

The central concept in the field of logic-based explanations for artificial neural networks (ANNs)—and more broadly for machine learning models—is the \emph{abductive explanation} (AXp)~\cite{DBLP:journals/ai/CooperS23}, which forms the foundation of the present work. This notion builds on prior theoretical research on \emph{abduction}~\cite{DBLP:conf/fair/Marquis91} and is grounded in the concept of the \emph{prime implicant} (PI). For this reason, it is also referred to as a \emph{PI-explanation}~\cite{DBLP:conf/ijcai/ShihCD18} or a \emph{sufficient reason}~\cite{DBLP:conf/ecai/DarwicheH20}. 
Abductive explanations have been applied to both tractable models, such as monotone and linear classifiers\cite{DBLP:journals/corr/abs-2008-05803,DBLP:journals/ai/CooperS23,DBLP:conf/kr/AudemardKM20}, and intractable ones, including random forests~\cite{DBLP:conf/ijcai/Izza021}, boosted trees~\cite{DBLP:conf/ijcai/AudemardBLM23}, and artificial neural networks~\cite{DBLP:conf/kr/ShiSDC20,ignatiev2019abduction}. 
In~\cite{DBLP:conf/kr/ShiSDC20}, binary neural networks are compiled into Ordered Binary Decision Diagrams (OBDDs), which are then used to compute AXps for the networks' classifications. In contrast,~\cite{ignatiev2019abduction} employ a \emph{Mixed Integer Linear Programming} (MILP) formulation to derive AXps for a neural network’s classifications in a three-digit MNIST task. Unlike our work and that of~\cite{DBLP:conf/kr/ShiSDC20},~\cite{ignatiev2019abduction} focus on neural networks with real-valued weights.  
Our approach differs from~\cite{ignatiev2019abduction} and~\cite{DBLP:conf/kr/ShiSDC20} in two key respects. First, causality plays a central role in our framework: we map a BSNN onto a binary causal model and leverage this causal representation to generate explanations. In contrast, neither~\cite{ignatiev2019abduction} nor~\cite{DBLP:conf/kr/ShiSDC20} incorporate any notion of causality. Second, they do not consider BSNNs, whereas BSNNs are the central focus of our analysis and the type of neural networks we aim to explain using logic and causal models.

Before concluding, it is worth mentioning the work on argumentation-based explanations of multi-layer perceptrons (MLPs) presented in~\cite{DBLP:conf/ecai/AyoobiPT23}. This approach builds on the mathematical relationships between MLPs and \emph{quantitative argumentation frameworks} (QAFs) established in~\cite{DBLP:conf/aaai/Potyka21}. The proposed method first sparsifies an MLP and then maps the resulting network onto an equivalent QAF, which can be used to explain and interpret the model's underlying mechanisms and decisions.   
Although this approach takes a different perspective, without an explicit grounding in logic or causality, we believe that a connection could be drawn between QAFs and causal models with continuous variables, and thus between the MLPs studied in~\cite{DBLP:conf/ecai/AyoobiPT23} and causal models. We leave this question for future work.





\section{Architecture, Learning  and Dataset }
\label{sec:archi}

In this section, we outline the details of the neural network models that we considered, along with the exact learning task, dataset and accuracies.

\subsection{Learning task}\label{subsec:learntask}


For our training purposes, we used the MNIST classification task for hand written digit recognition. We trained
networks with a single fully
connected hidden layer on both tasks, 3-digit 
and 10-digit MNIST classification. As we will show in Table \ref{table1}, we could achieve very high accuracy with
binary quantized 
networks  on the 3-digit classification task.
We could also  achieve a high accuracy on the 10-digit classification task with three-value quantized  networks with weights ranging over 
$\{-1, 0, 1\}$. 

\subsection{Spike encoding} \label{spikEnc}
For our experiments, we used two different approaches to convert MNIST images into spikes. Firstly, we used a classic Poisson rate coding scheme \cite{rate_coding} to convert images into spike trains in multiple time-steps and also a threshold-binarized scheme with just one time-step as presented in Table \ref{table1}.
We did not pursue temporal coding in our experiments 
since, as shown by \cite{bs4nn},
temporal coding
requires  larger time-steps for training with high accuracy. Since having more time-steps significantly increases the complexity of finding an explanation, we chose to not use temporal coding in this work. 
Nonetheless, the  novel mapping
of BSNNs to binary causal models
we will present 
can be generalized to  to other forms of spike encodings. 
 We used a simple Integrate and Fire (IF) model for our spiking neurons, since 
 mapping BSNNs to binary causal
 models
 is easier in the absence of leaks.

\subsection{Weight quantization} \label{weight} 

As we will demonstrate  later  in the paper, mapping a BSNN to  a binary  causal model requires the network to have weights quantized either in a binary
(i.e., $\{0, 1\}$)  or  a 
three-valued (i.e., $\{-1, 0, 1\}$)  way. To train our networks, the weight quantization procedure that we adopted closely follows the XNOR-Net proposal by \cite{10.1007/978-3-319-46493-0_32}, i.e., during a forward pass the network uses a binarized weight matrix $\mathcal{B}(W)$, while during the backward pass it retains a proxy full-precision weight matrix $W$ for gradient calculation. Straight-through-estimator (STE) \cite{bengio2013estimating} was used without any gradient clipping for our training. 
The following equations
represent the two variants of the quantizing functions $\mathcal{B}^{bin} $ and $\mathcal{B}^{tern}$ we used:
\begin{align}
       \mathcal{B}^{\mathit{bin}}(W_{i,j})=& 
        \begin{cases}
           & 0, \text{if} \ W_{i,j} = 0, \\
      & (\textit{sign}(W_{i,j}) + 1)/2, \text{if} \ W_{i,j} \neq 0 , 
        \end{cases}\\
 \mathcal{B}^{\mathit{tern}}(W_{i,j})= & \ \ \textit{sign}(W_{i,j})   , \label{funct:tern}
   \end{align} 
with $W_{i,j}$ the $(i,j)$-coordinate of the weight matrix $W$.  
In order to train our networks through standard back-propagation based methods for supervised learning, we employed a surrogate gradient descent approach \cite{neftci2019surrogate} with \emph{arctan} as the surrogate function along with a STE  for updating binary weights \cite{bengio2013estimating}, in a way similar to  \cite{Jang2020BiSNNTS}.


\section{Formal Model of Spiking Neurons}\label{sec:modeldyn}


In this section, we  introduce the 
formal 
model of
a   binary spiking neural network (BSNN)
 and of its integrate-fire (IF) spiking dynamics. Spiking neurons have the ability to process rich temporal dynamics in the data due to the state fullness of the neurons much like in recurrent neural networks (RNNs). We first introduce the static architecture of a BSNN. 

\begin{definition}[BSNN architecture] \label{defAlternative}
The architecture of a BSNN is a tuple
   $S = \big(\mathbf{I}, \mathbf{L}, 
\mathcal{R},\mathcal{W},
\mathit{Scale} 
,(\tau_X )_{ X \in  \mathbf{L} } 
\big)$ 
where:
\begin{itemize}
    \item $\mathbf{I}$
    and $\mathbf{L}$
    are two non-empty  disjoint sets, respectively, the
    set of input neurons and the set of non-input neurons,
    with  $\mathbf{N}= \mathbf{I} \cup \mathbf{L} $ (the set of neurons); 
   \item
   $\mathcal{R}\subseteq  \mathbf{L} \times \mathbf{N} $
   is a connectivity relation 
   relating each non-input  neuron
   to its predecessors; 
   \item $\mathcal{W} :\mathcal{R} 
\longrightarrow \mathit{Scale}$
is the weighting  function
for the connectivity relation,  with
$\mathit{Scale}$
a finite scale of integers
(i.e., 
$\mathit{Scale}=
 \{k_1,  \ldots, k_n\} \subset \mathbb{Z}$
 such that 
 $
 k_1 < \ldots < k_n 
$); 
\item $\tau_X$
is the firing threshold for the non-input  neuron  $X \in \mathbf{L}$. 
\end{itemize}
\end{definition}

Given the architecture of a BSNN, 
we introduce the following notion
of BSNN-compatible fire  spiking dynamics. 
\begin{definition}[BSNN-compatible fire spiking dynamics] 
\label{defDynamics}
Let 
$S = \big(\mathbf{I}, \mathbf{L}, 
\mathcal{R},\mathcal{W},
\mathit{Scale} 
,
(\tau_X )_{ X \in  \mathbf{L} } 
\big)$
be the architecture of a BSNN
and let 
$F= (\mathcal{F}_X)_{X \in \mathbf{N}}  $
be a family
of firing functions for $S$'s
neurons,
with 
$\mathcal{F}_X  : \mathbb{N} \longrightarrow \{0,1\}$.
We say that $F$
represents
a possible
spiking dynamics 
for the BSNN $S$ up to time $\mathsf{t}_{\mathit{end}}\geq 0$,
or simply $F$ is  $S$-compatible up to time $\mathsf{t}_{\mathit{end}}$,  
if and only if
the following condition holds for every $X \in \mathbf{L}$
and for every $t\leq \mathsf{t}_{\mathit{end}}$:
   \begin{align}
       & 
       \mathcal{F}_X (t)= 
        \begin{cases}
           & 0, \text{ if }t = 0, \\
      & \Theta\big(    
    \mathcal{A}(X,t) - \tau_X 
       \big), \text{ if } t > 0 , 
        \end{cases}
        \label{funct:firing}
   \end{align}
where 
\begin{align*}
\mathcal{A}(X,t)& =
 \begin{cases}
0, \text{ if } t = 0, \\
\mathcal{A} ( X,t-1) \cdot \big(1 - \mathcal{F}_X(t-1) \big) \\ +
       \sum_{
       (X, X' ) \in \mathcal{R}   }
       \mathcal{W}(X,X' ) \cdot 
       \mathcal{F}_{X'}(t), \text{ if } t > 0,
\end{cases}
\end{align*}
and
\begin{align*}
  \Theta(x) & = 
    \begin{cases}
        1,  \text{ if } x \geq 0, \\
        0, \text{ otherwise}. 
    \end{cases}
\end{align*}
\end{definition}
Some explanations of the previous two definitions are in order.
The weighting function $\mathcal{W}$
in  Definition \ref{defAlternative}
specifies
for each  non-input neuron $X \in \mathbf{L}$
and each   
predecessor $X' \in  \mathcal{R}(X) $
the weight of the 
connection from
$X'$ to $X$,
with $\mathcal{R}(X) =\big\{  X' \in 
 \mathbf{N} : (X,X') \in\mathcal{R}  \big\}$.
In the general model,
a weight can take any value from
the set of numerical values $\mathit{Scale}$.
In  the rest of our paper we will  only consider the 
BSNN
variants
of the model with 
$\mathit{Scale}= \{-1, 0, +1\}$
or 
$\mathit{Scale}= \{ 0, +1\}$.
From a mathematical point of view,
BSNNs are nothing but special cases of SNNs with either Boolean 
or three-valued weights. 

Note that by means of the connectivity
relation $\mathcal{R} $ we can specify
the set of output neurons 
$\mathbf{O}$ as the non-input neurons
that have no successors, that is,
\begin{align*}
    \mathbf{O} = \big\{ X \in 
     \mathbf{L} : 
       \forall   X'  \in \mathbf{L},  (X', X) \not \in \mathcal{R}
     \big\}. 
\end{align*}

Definition
\ref{defDynamics}
describes  the possible
 spiking dynamics
of a BSNN   $S$.
In particular, 
 the firing  function $\mathcal{F}_X$  represents
a possible dynamics of the non-input neuron $X$ in the BSNN architecture: it is  the Heaviside step function
of the difference between the neuron's activation value  
 and the spiking threshold $\tau_X$. 
The firing activity of the input neurons does not depend on the firing  activity of other neurons, it is  uniquely determined by the temporally sequential presentation of the input. This is the reason why the condition for $\mathcal{F}_X$ only applies to the case $X\in \mathbf{L} $. 

The activation value of the non-input   neuron $X$ at time $t$ 
depends recursively on its value at time $t-1$ and a weighted sum over the incoming stimulus at time $t$. Therefore, to respect the recursive nature of the activation function, we have to define that at time $0$, the network is completely inactive, i.e.,  no node $X \in \mathbf{N}$ is firing at time $t=0$. 
Moreover, the incoming stimulus gets perfectly integrated as in an Integrate-Fire (IF) model, without any leak in the neurons. But there is a hard reset term in our neuron model, which resets the activation value to zero every time it fires a spike.

 The BSNN architectures we trained for the MNIST classification
 task we informally described above 
 are  specific instances  of Definition \ref{defAlternative}. Specifically,
 each network has $28 \times 28$
 input neurons, one neuron per pixel in the image to be classified.
 That is,
$  
   \mathbf{I}  =   \big\{ \instnot{I}_{x,y} : 1 \leq x,y \leq 28 \big\}$.
   
   Moreover, it has either 8, 16, 32, 64 or 128 hidden neurons
   in the hidden   layer
   that are fully connected to the input neurons, that is, 
   given  $k\in \{8,16,32,64,128\}$:
$     
   \mathbf{H}^{} =   \big\{ \instnot{H}_{z} : 1 \leq z \leq k \big\}$, and 
$
\forall \instnot{H}     \in \mathbf{H} ,
\forall \instnot{I}    \in \mathbf{I} ,
(  \instnot{H} , \instnot{I} ) \in \mathcal{R}$. 

   Finally, it has 10 classification  neurons
   in the output layer, one neuron for each digit
   to be recognized in the general
   MNIST classification task
      that are fully connected to the hidden  neurons, 
   that is,  
$
   \mathbf{C}   =  
   \big\{ \instnot{C}_{z} : 1 \leq z  \leq 10  \big\} $,
and $
\forall \instnot{H}_z   \in \mathbf{H} ,
\forall \instnot{C}_{z' }   \in \mathbf{C}  , 
(\instnot{C}_{z'}, \instnot{H}_z  ) \in \mathcal{R} $. 

Thus,  we considered a BSNN with a set of non-input  neurons  $ \mathbf{L}= \mathbf{H}  \cup  \mathbf{C}$. 
Notice that
in this BSNN architecture 
the set of classification neurons coincides
with the set of output neurons, that is, $   \mathbf{O}=
  \mathbf{C}$.

  \begin{table*}[t]
\centering
\resizebox{1.2\textwidth}{!}{%
\begin{tabular}{|c|c|c|c|c|c|c|}
\hline
Model type & Number of hidden neurons (k) & Digits & Spike encoding & Time-steps ($t_{end}$) & Validation Accuracy (\%) & Test Accuracy (\%) \\ \hline
  & 32 & & Poisson & 16 & 92.98 & 94.29 \\
  & 16 & & Poisson & 16 & 94.68 & 94.62 \\
  & 8  & & Poisson & 8  & 95.20 & 95.27 \\
  & 32 & & Thresholded & 1 & 92.47 & 93.63 \\
  & 16 & & Thresholded & 1 & 92.09 & 91.66 \\
  \multirow{-6}{*}{$\mathscr{S}_{k}^{\mathit{bin}}$} & 8 & \multirow{-6}{*}{1,5,9} & Thresholded & 1 & 91.29 & 93.41 \\ \hline
  \multirow{6}{*}{$\mathscr{S}_{k}^{\mathit{tern}}$} & 128 & \multirow{6}{*}{0,1,2,3,4,5,6,7,8,9} & Poisson & 4 & 92.00 & 92.16 \\
  & 64 & & Poisson & 4 & 91.82 & 92.03 \\
  & 32 & & Poisson & 4 & 90.55 & 91.06 \\
  & 128 & & Thresholded & 1 & 86.56 & 87.00 \\
  & 64 & & Thresholded & 1 & 84.97 & 86.10 \\
  & 32 & & Thresholded & 1 & 85.12 & 85.03 \\ \hline
\end{tabular}%
}
\caption{Accuracies of different BSNN architectures trained on the MNIST classification task.}
\label{table1}
\end{table*}

 The class of 
BSNN architectures  
  with binary weights
  are denoted by
    $\mathscr{S}_{k}^{\mathit{bin} }$  while
 those 
  with three-valued  weights
  are denoted by
    $\mathscr{S}_{k}^{\mathit{tern} }$, 
    depending
    on the number $k$ of their hidden units. 
We  only trained  and tested 12 variants of 
BSNN networks
varying along the three dimensions:
the specific
spike encoding used (Poisson vs.
threshold binarized), 
as detailed above, the weight quantization
used ($\{0,1\}$ vs. $\{-1, 0,1\}$), and the number $k\in \{8,16,32, 64, 128\}$ of hidden
units.
For each  variant, 
the value of $\mathcal{W}(X,X') $
for each $(X,X') \in\mathcal{R} $
was determined through  learning.
Specifically,
we have three networks
for each of the following four cases: i) 
  binary weights,   Poisson encoding and $k\in \{8,16,32\}$; 
  ii)     
  binary weights,   threshold binarized encoding and
  $k\in \{8,16,32\}$; 
    iii)  
  three-valued  weights,   Poisson  encoding and
  $k\in \{32,64,128\}$; 
      iii)    
  three-valued  weights,   threshold binarized encoding and
  $k\in \{32,64,128\}$.



\section{Causal Model} \label{causalMod}


 A  causal
 model
 is a mathematical
 object  describing the causal
 dependencies between variables.
 It is
a central concept of  current analyses  of causality
 in AI. 
 A binary causal model (BCM)
 is nothing but a causal
 model
 in which variables are assumed to be Boolean. 
 In a BCM
 causal information
 is expressed by means of Boolean
 expressions (\emph{alias}
 propositional formulas),
 the set of 
  Boolean expressions being  generated inductively as follows:
i)  each Boolean variable $p$
is a Boolean expression;
ii) if $\omega$
is a Boolean expression, so 
is $\neg \omega$ (``negation''); 
iii) if $\omega_1$
and $\omega_2$
are Boolean expressions,
so is $\omega_1\wedge \omega_2$
(``conjunction''). Additional
Boolean 
constructs
$\top$, $\bot $,
$\vee$, 
$\rightarrow$
and $\leftrightarrow$
are definable as abbreviations in the usual way. 
In formal terms, a BCM  is a triplet 
$\Gamma = (\mathbf{U}, 
\mathbf{V},\mathcal{E })$
where i) $\mathbf{U}$
    is a set of exogenous  variables, 
    ii) $\mathbf{V}$
    is a set of endogenous  variables, 
    iii) $\mathcal{E }$
    is a function mapping  each 
   endogenous variable $p  \in \mathbf {V}$
    to a Boolean expression $\mathcal{E }(p) $ of the form  
   $p \leftrightarrow \omega_p$, 
where 
$\omega_p $
is a Boolean expression built
from $\mathbf{U}\cup 
\mathbf{V}$ that does not contain $p$. 
Specifically, the Boolean expression  $  p \leftrightarrow \omega_p$
stipulates  that the endogenous variable  $p$ is true iff the condition 
$\omega_p$
is true. It can be seen as the compact representation
of a Boolean function for the endogenous variable  $p$. 
From a binary causal model 
$\Gamma = (\mathbf{U}, 
\mathbf{V}, \mathcal{E } )$ it is straightforward to extract a causal graph representing
the causal dependencies between the variables:
the vertices of the causal graph are the variables in
$\mathbf{U}\cup 
\mathbf{V}$, and we draw an edge from a
variable $q$
to an endogenous variable $p$
if the Boolean
expression $\omega_p$
such that $\mathcal{E }(p)= 
p \leftrightarrow \omega_p $
  contains the variable $q$.

The model of the BSNN 
given in Definition \ref{defAlternative}
can be mapped onto a BCM 
that represents the causal dependencies between the BSNN's neural units over time. The idea of the mapping is simple: we assign a Boolean variable $p_{X,t}$ to each neuron $X$ for each time $t$ in $\{0,\ldots, \mathsf{t}_{\mathit{end}}\}$,
where  $\mathsf{t}_{\mathit{end}}$ 
is the final time step at which the network stops receiving incoming spike train from the image currently being presented. 
The variable $p_{X,t}$ is true (resp. false)
if the neuron $X$
fires (resp. does not fire) at  time $t$. 
The  exogenous variables are for the input
neurons, while the endogenous ones are for the non-input  neurons. 
The causal dependencies between the firing activities of the neurons are represented by the Boolean equations. Here, 
we only give the BCM
for the variants of the BSNN
with Boolean weights
$\{0,1\}$. 
\begin{definition}[BCM  for BSNN with Boolean weights]\label{def:causalmodel}
Let 
$S = \big(\mathbf{I}, \mathbf{L}, 
\mathcal{R},\mathcal{W},
\{0,1\} 
,
(\tau_X )_{ X \in  \mathbf{L} }
\big)$
be the architecture of a BSNN with Boolean weights
in the sense of Definition \ref{defAlternative}.
The BCM for  $S$
is the triplet 
$\Gamma_S  = \big(\mathbf{U}_S, 
\mathbf{V}_S, \mathcal{E }_S
\big)$
where
$   \mathbf{U}_S=  \bigcup_{0 \leq t \leq
\mathsf{t}_{\mathit{end}}
}  \mathbf{U}_S^t$, $   \mathbf{V}_S=  \bigcup_{0 \leq t \leq
\mathsf{t}_{\mathit{end}}
}  \mathbf{V}_S^t$,
$\mathbf{U}_S^t = \{ p_{X,t} : X \in \mathbf{I} \}$, 
$ \mathbf{V}_S^t = \{ p_{X,t} : X \in \mathbf{L} \}$,
and 
$\forall X \in \mathbf{L} $:
  {\footnotesize
 \begin{align*}
     \mathcal{E }_S(  p_{X,0} ) :=
    p_{ X , 0 } \leftrightarrow   \bot ,   
 \end{align*}
 }
     and
  for   $t>0$:
  {\footnotesize
\begin{align*}
     \mathcal{E }_S(  p_{X,t} ) :=
              p_{ X , t } 
          \leftrightarrow  
 \Bigg(  & \Big( \neg p_{ X , t-1} 
       \rightarrow 
    \bigvee\limits_{ \substack{ \Omega \subseteq \mathcal{R}^+(X):   \\  
      \mathcal{A}(X,t-1)  
    +|\Omega | \geq   \tau_X }} \big(  \bigwedge\limits_{X' \in \Omega } p_{X', t } \big) \Big)  \notag \\
&      \wedge  \Big(   p_{ X , t-1} \rightarrow 
    \bigvee\limits_{ \substack{ \Omega  \subseteq \mathcal{R}^+(X) : 
     \\     |\Omega | \geq   \tau_X }} \big(\bigwedge\limits_{X' \in \Omega } p_{X', t }\big) \Big) \Bigg),  \label{equationcausal2}
\end{align*} 
}
\noindent with  
\begin{align*}
    \mathcal{R}^+(X)= \big\{ X' \in \mathbf{N}
: (X,X') \in \mathcal{R} 
\text{ and } \mathcal{W}(X,X')=
1
\big\}. 
\end{align*}
\end{definition}
We conclude this section by showing that 
the  spiking dynamics of a BSNN 
are correctly represented by its  BCM. Specifically, 
let 
$S = \big(\mathbf{I}, \mathbf{L}, 
\mathcal{R},\mathcal{W},
\{0,1\} 
,(\tau_X )_{ X \in  \mathbf{L} }
\big)$
be a BSNN with Boolean weights 
and 
$\mathcal{I }$
a Boolean  interpretation
for the variables in $\mathbf{U}_S \cup \mathbf{V}_S$,
i.e., 
$\mathcal{I } :
\mathbf{U}_S \cup \mathbf{V}_S
\longrightarrow \{0,1\}
$, 
such that
for every time $t\in \{0,\ldots, \mathsf{t}_{\mathit{end}}\} $
and for every neuron $X$, 
the function 
$\mathcal{F}_X$ assigns
to time  $t$
the same
value
assigned by 
 the interpretation 
$\mathcal{I}$ to the corresponding  variable
$p_{X,t}$.
Then, 
the  family of firing
functions 
$ F=(\mathcal{F}_X)_{X \in \mathbf{N}  }$
is 
$S$-compatible up to time $\mathsf{t}_{\mathit{end}}$
if and only if 
$\mathcal{I }$
satisfies
all Boolean equations
of the BCM
$\Gamma_S  = \big(\mathbf{U}_S, 
\mathbf{V}_S, \mathcal{E }_S
\big)$
for  $S$.
This 
correspondence between a BSNN and 
its  BCM  
is formally expressed by the following
Theorem \ref{theo:character} where, for any  Boolean expression $\omega $, 
$ \mathcal{I } \models \omega $
denotes the fact that the Boolean interpretation 
$  \mathcal{I }$
satisfies the Boolean expression $\omega$.
For the readers unfamiliar with Boolean (propositional) logic, 
we remind that 
$ \mathcal{I } \models \omega $
iff $\mathit{Val}( \mathcal{I }, \omega )=1 $, 
where 
$\mathit{Val}( \mathcal{I }, \omega )$
is   defined inductively, as follows:
i) 
$\mathit{Val}( \mathcal{I }, p )= \mathcal{I }(p)
\text{ for } p \in (\mathbf{U}_S \cup \mathbf{V}_S) $;
ii)  $ \mathit{Val}( \mathcal{I },  \neg \omega )=
1 -\mathit{Val}( \mathcal{I },  \omega )$;
iii)  $ \mathit{Val}( \mathcal{I },   \omega_1\wedge \omega_2 )=
\min \big(\! \mathit{Val}( \mathcal{I },   \omega_1 ){,} 
\mathit{Val}( \mathcal{I },    \omega_2 ) \big) $. 
\begin{theorem}\label{theo:character}
Let $ \mathcal{I }(p_{X,t })= \mathcal{F}_X(t)
$ for all 
$  X \in \mathbf{N}$
and 
for all 
$  t \leq  \mathsf{t}_{\mathit{end}}$.
Then, the following  are equivalent:
\begin{itemize}
    \item $ (\mathcal{F}_X)_{X \in \mathbf{N}  } $
  is 
  $S$-compatible up to time $\mathsf{t}_{\mathit{end}}$, 
  
    \item $ \mathcal{I }
  \models \bigwedge_{ p_{X,t} \in  \mathbf{V}_S}
  \mathcal{E}_S(p_{X,t})$. 
\end{itemize}

\end{theorem}
The proof of the theorem is given 
in the appendix \ref{bcm:proofs1}
at the end of the paper.

\section{Explanation}\label{sec:expl}

In this section,
we are going to show 
how to use 
binary causal
models (BCMs) for 
formalizing  and computing explanations
in the context of the BSNN
architectures
we trained 
for the MNIST classification
task. Following the literature
on abductive explanation (AXp) \cite{ignatiev2019abduction,DBLP:journals/logcom/LiuL23},
we define it to be a prime implicant
that is actually true.
Moreover, we define it in relation
to a binary causal model.
For simplicity, we assume an AXp (the \emph{explanans}) 
is a term 
  made
of exogenous variables and the property
to be explained (the \emph{explanandum}) is a Boolean
expression made of endogenous ones. This
assumption is perfectly
compatible with our application to the MNIST
classification task in which
we want to explain the network classification on the basis
of the pixel-level features. Nonetheless, this assumption could be dropped without consequence; we would only need to assume that the explanans and the explanandum involve different variables. 

Some preliminary notions
are needed before defining AXp formally.
We define a \emph{term}
to be 
a conjunction of literals  in which a   variable
can occur at most once, a literal being
a variable $p $
or its negation $\neg p$.
Terms are denoted by
$\lambda, \lambda', \ldots$
Given two terms $\lambda , \lambda'  $,
with a bit of abuse of notation,
we write $\lambda' \subseteq \lambda$
(resp. $\lambda' \subset \lambda $)
to mean that the set of literals
appearing in $\lambda'$
is a subset (resp. strict subset)  of the set of literals
appearing in $\lambda$. 
Given a 
 BCM $\Gamma = \big(\mathbf{U}, 
\mathbf{V},\mathcal{E }
\big)$
and an arbitrary set of variables 
$\mathbf{X }\subseteq \mathbf{U}\cup \mathbf{V}$, $\mathit{Term}_{\mathbf{X}}$
denotes 
the set of terms built from
$\mathbf{X }$.

\begin{definition}[Abductive explanation]\label{def:axpdef}
Let $\Gamma = \big(\mathbf{U}, 
\mathbf{V},\mathcal{E }
\big)$
be a BCM,  $\mathcal{I }_{\mathbf{U} } :
\mathbf{U}  
\longrightarrow \{0,1\}
$
a 
Boolean
interpretation for
its exogenous   variables,
$\lambda \in \mathit{Term}_{\mathbf{U}} $  
and $\omega_0$
a Boolean expression 
built from $   \mathbf{V}$. 
We say that $\lambda$
is an abductive explanation
(AXp) of
$\omega_0$
with respect  to $\Gamma $ and  $\mathcal{I }_{\mathbf{U} } $
if and only if:\label{AXp}
\begin{align*} 
 i)  & \ \mathcal{I }_{\mathbf{U} }  \models \lambda  ,\\
    ii)  &
    \models
    \big(\bigwedge_{p \in \mathbf{V} }
    \mathcal{E}(p) \wedge 
    \lambda \big)  \rightarrow \omega_0,\\
        iii) & \
        \forall \lambda' \subset \lambda, \not
        \models  \big(\bigwedge_{p \in \mathbf{V} }
    \mathcal{E}(p) \wedge 
    \lambda'  \big) \rightarrow \omega_0 ,
\end{align*}
where, for a given Boolean expression $\omega$ built from the set of variables $\mathbf{U} \cup
\mathbf{V}$, $\models \omega$
means that $\omega $
is valid, i.e., 
$\mathcal{I}\models \omega$
for every Boolean interpretation
$\mathcal{I} \in \{0,1\}^{\mathbf{U} \cup
\mathbf{V} } $.
\end{definition}
Let us illustrate how Definition \ref{def:axpdef}
applies to  a BSNN $S $ with Boolean weights 
(i.e., $S \in \mathscr{S}_{k}^{\mathit{bin} }$)
trained on the MNIST three-digit classification task. 
Given  an input sequence 
$\mathit{input}:\{0,\ldots, \mathsf{t}_{\mathit{end}}\} \times \mathbf{I} \longrightarrow \{0,1\} $
and an observed  output sequence 
$\mathit{out}:\{0,\ldots, \mathsf{t}_{\mathit{end}}\} \times \mathbf{C} \longrightarrow \{0,1\} $
for this input, 
we  aim to abductively explain   
 the output at a chosen  time  $t \in \{0,\ldots, \mathsf{t}_{\mathit{end}}\}$
using only variables for  the input at time $t$. 
More precisely,
we take  the \emph{explanandum} (i.e., $\omega_0$)
to be
the Boolean expression 
\begin{align*}
   \mathsf{out}_{S,t }  =_{\mathit{def}}
\bigwedge_{
\substack{ 
\instnot{C}_z \in \mathbf{C}:\\ \mathit{out}(t, \instnot{C}_z)=1 } }  p_{\instnot{C}_z ,t   } \wedge \bigwedge_{ \substack{  \instnot{C}_z
\in \mathbf{C}
:\\ \mathit{out}(t, \instnot{C}_z)=0 } }  \neg p_{\instnot{C}_z ,t   } . 
\end{align*}
It 
represents the observed 
output of the network at time $t$.
Then,   we  search  for an abductive explanation 
$\lambda \in \mathit{Term }_{\mathbf{U}_{S}^t }$
of $\mathsf{out}_{S,t }  $
with respect 
to the BCM $\Gamma_{S} $ 
and to the Boolean interpretation $\mathcal{I}_{ \mathbf{U}_{S} }$
encoding the input sequence $\mathit{input}$
(i.e.,  $\mathcal{I}_{ \mathbf{U}_{S} }  ( p_{ \instnot{I  }_{x,y}, t  }       )=
\mathit{input}(t,\instnot{I  }_{x,y} )$
for every $t \in\{0,\ldots, \mathsf{t}_{\mathit{end}}\}  $
and $\instnot{I }_{x,y} \in \mathbf{I }$). 
The latter
condition guarantees that
the found explanation 
of the network's output at time $t$ 
represents 
a portion of 
the actual input presented to  the network at  $t$. 

The following proposition  highlights an important property
of  a BSNN's abductive explanation:
any input feature/neuron    being mentioned in an abductive explanation 
of the output has  necessarily 
a non-zero weight connection with  the network's hidden layer.
This 
guarantees that an abductive explanation does not contain completely 
irrelevant
information. Later in the paper, 
we will contrast this result
with the SHAP explanation method
for which there is no guarantee that a found explanation
does not contain completely  irrelevant information. 
\begin{proposition}\label{prop:existencelink}
    Let $\lambda \in \mathit{Term }_{\mathbf{U}_{S }^t  }$
    be  an abductive explanation of 
$ \mathsf{out}_{S,t } $. Then, 
\begin{align*}
    \forall p_{\instnot{I}_{  }, t } \subseteq \lambda,
    \exists \instnot{H}_{ } \in \mathbf{H}  \text{ such that }
   \instnot{I}  \in  \mathcal{R}^+ ( \instnot{H}_{ }), 
\end{align*}
where we recall 
$\mathcal{R}^+( \instnot{H}_{ })= \big\{  \instnot{I}_{  } \in  \mathbf{I} 
: (\instnot{H}_{ }, \instnot{I}_{ }) \in \mathcal{R} 
\text{ and } \mathcal{W}(\instnot{H}_{ }, \instnot{I}_{  }) =
1
\big\}$. 
\end{proposition}
The proof of the proposition is given 
in the Appendix  \ref{bcm:proofs2}
at the end of the paper.


To compute an abductive explanation, we rely
on a standard abductive explanation search
algorithm, whose pseudo code is presented in Algorithm \ref{alg:AXp}. 
The algorithm is initialized with a complete term $\lambda_{\mathit{init}}$ over the set of exogenous variables (i.e., $\mathbf{U}_{S }$), which fully represents the actual input at the selected time $t$. Then, literals are systematically removed from $\lambda_{\mathit{init}}$, and at each iteration, we check whether condition (ii) in Definition \ref{AXp} is still satisfied.

    \begin{algorithm} 
    \caption{Computing Abductive Explanation}\label{alg:AXp}
    \begin{algorithmic}
    \Require Initial implicant $\lambda_{init}$ and explanandum $\omega_0$ that satisfy conditions (i) and (ii) in Def. \ref{AXp}
    \Ensure Abductive explanation $\lambda$
    \State Set $\lambda$ = $\lambda_{init}$
    \For{\textit{l} $\in \lambda$}
        \If{ $\models \big(\bigwedge_{p \in \mathbf{V} }
    \mathcal{E}(p) \wedge 
    \lambda \big)  \rightarrow \omega_0$ }
            \State $\lambda \rightarrow \lambda \setminus \textit{l}$
        \EndIf
    \EndFor
    \State \Return{$\lambda$}
    \end{algorithmic}
    \end{algorithm}

At the end of the search algorithm we further verify
the validity of condition (iii) in Definition \ref{AXp} for a \emph{prime implicant} check of the resulting abductive explanation  $\lambda$.
Algorithm \ref{alg:AXp} has a time complexity of $\mathcal{O}(|\mathbf{U}_{S_k^{bin}}|)$ which is the total number of exogenous variables in the model. This linear dependency guarantees the scalability of the algorithm with respect to the number of input neurons. 

\section{Experimental Results}\label{sec:results}

In this section, we provide the experimental results 
on computing explanations  for some 
of the BSNNs listed  in Table \ref{table1}. 
We implemented the AXp search Algorithm \ref{alg:AXp} using the open-source $Z3$ solver, which is an efficient and flexible theorem proving system
 implemented in Python developed by Microsoft Research.
Since the time required to compute explanations using the SAT-based Algorithm 1 was very high (on the order of hours), we also considered a variant of this algorithm based on an SMT (Satisfiability Modulo Theories) representation of the binary causal model for the BSNNs and of the notion of abductive explanation. (See \cite{BSST09} for a general
introduction to SMT.) In particular,
  SMT over Linear Integer
Arithmetic (LIA) was sufficient for our purpose. 
We could again use Z3
given that it   fully supports SMT.
The exact representation can be found  in the 
Appendix \ref{appendix:smt}
at the end of the paper.


\begin{table}[h]
    \centering
    \begin{subtable}{0.74\textwidth}
        \centering
        {\scriptsize
        \resizebox{\linewidth}{!}{%
        \begin{tabular}{|c|cc|cc|}
            \hline
            \multirow{2}{*}{Number of hidden neurons ($k$)} & \multicolumn{2}{c|}{Mean search time} & \multicolumn{2}{c|}{Length of found explanation} \\ \cline{2-5}
             & SAT (hrs) & SMT (s) & (\%) Total features & Mean \\ \hline
            32 & 10.7 & 491 & 20.91 & 164 \\
            16 & 5.84 & 483 & 27.3 & 214 \\
            8  & 11.13 & 192 & 12.5 & 98  \\ \hline
        \end{tabular}
        }
        }
          \caption{Results for classes  $\mathscr{S}_{k}^{\mathit{bin}}$}
        \label{tab:subtable1}
    \end{subtable}
    \begin{subtable}{\textwidth}
        \centering
        {\scriptsize
        \begin{tabular}{|c|c|cc|}
            \hline
            \multirow{2}{*}{Number of hidden neurons ($k$)} & Mean search time  & \multicolumn{2}{c|}{Length of found explanation} \\ \cline{3-4}
             & SMT (hrs) & (\%) Total features & Mean \\ \hline
            128 & 0.27  & 56 & 437 \\
            64  & 0.78  & 55 & 432 \\ 
            32  & 1.0   & 36 & 280 \\
            \hline
        \end{tabular}
        }
        \caption{Results for classes  $\mathscr{S}_{k}^{\mathit{tern}}$}
        \label{tab:subtable2}
    \end{subtable}
    \caption{Computational analysis for searching explanation 
    of BSNNs in the classes 
    $\mathscr{S}_{k}^{\mathit{bin}}$ and $\mathscr{S}_{k}^{\mathit{tern}}$.}
    \label{fig:combined_tables}
\end{table}

Table \ref{tab:subtable1}  provides a comprehensive overview of the run-times for the SAT-based and SMT-based versions of the explanation search algorithm, along with the length of the AXp found for each BSNN in the class  $\mathscr{S}_{k}^{\mathit{bin}}$ listed in Table \ref{table1}, with $k \in \{8, 16, 32\}$. 
As the table clearly shows,   the SMT-based
approach   is 
is significantly faster than
the SAT-based approach. 
This is because, unlike the propositional logic representation of the binary causal model, the SMT representation avoids the need for universal quantification over sets of variables. We also computed explanations for BSNNs in the class $\mathscr{S}_{k}^{\mathit{tern}}$, using the SMT representation of their binary causal models and the corresponding notion of abductive explanation. Table \ref{tab:subtable2} reports the run-times for computing these explanations with the SMT-based approach, along with their average length. 

Figure \ref{AXpVisu} visualizes the abductive explanation found for the outputs of a network in the class $\mathscr{S}_{16}^{\mathit{bin}}$ at times 0 and 6. 

 Note that the set of 
 input neurons/features 
in the explanation
 is a subset of 
 the set of input neurons/features
 connected to the network's hidden layer.
 This is 
   in line with Proposition \ref{prop:existencelink}.
 \begin{figure}[H]
 \centering
     \includegraphics[width=\columnwidth]{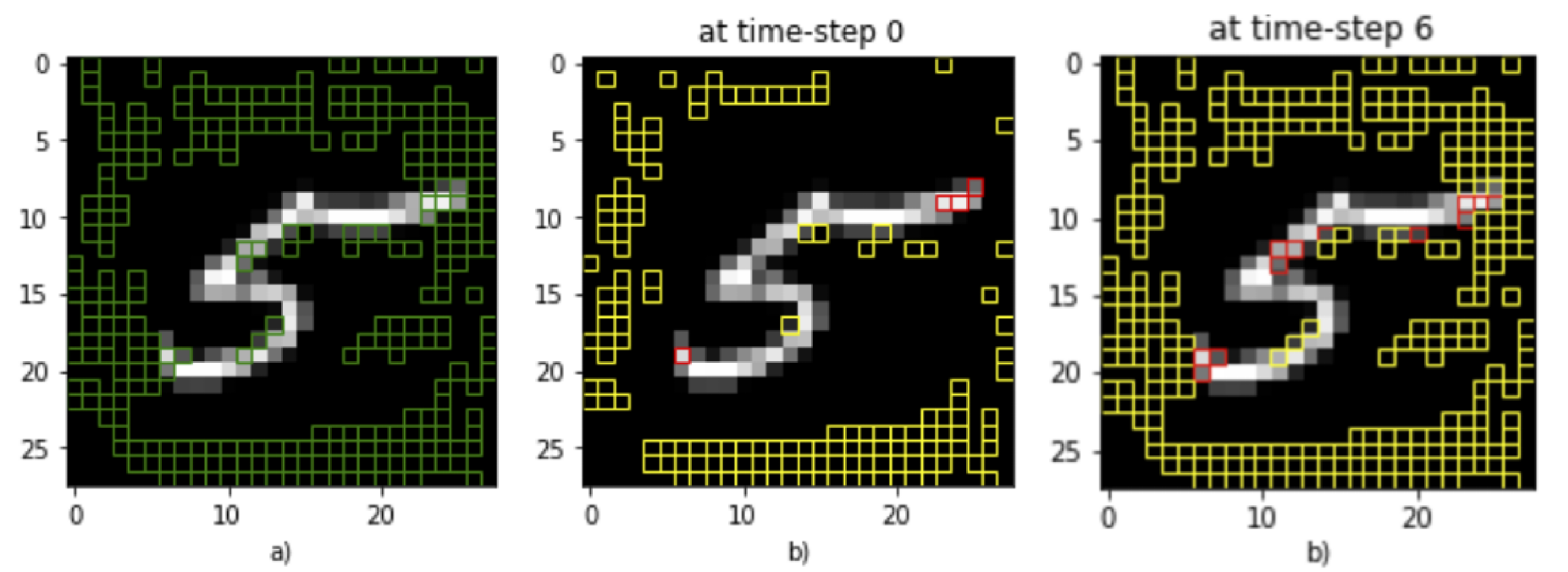} 
    \caption{Image of digit 5  (a) showing in green the input neurons/features
    being connected with the network's hidden layer; (b) the found AXps at times  $0$ and $6$ showing in red the active input  neurons/features
    (i.e., the positive literals)
 and in yellow
the 
    non-active input neurons/features
    (i.e., the negative literals)    mentioned in the explanation. }\label{AXpVisu} 
 \end{figure}
 
 \section{Comparison with SHAP}\label{sec:shap}

In this section, we  compare  our logic-based explainability  method 
with SHAP, a popular   method 
 widely used for interpreting
 predictions
 of machine learning models \cite{nips_shap}. 
For our experiments, we used the pre-existing implementation of 
SHAP library in Python
available at \url{https://github.com/shap/shap}. SHAP  assigns relevance scores to input features based on a sample of the input space 
without taking into consideration  the internal dynamics of the model.

\begin{table}[h]
\centering
{\footnotesize 
\begin{tabular}{|c|c|c|}
\hline
Sample size & Mean  & Features wrongly    \\ 
 & computation time (s) & considered relevant (\%) \\ \hline
1000000 & 173.6 & 36.95 \\
100000 & 38.3 & 46.34 \\
10000 & 4.7 & 57.45 \\ \hline
\end{tabular}
}
\caption{Percentage of features wrongly considered relevant by SHAP.}
\label{table3}
\end{table}

Unlike our method, SHAP does not look inside the neural network and does not model
the network's internal causal structure.  
Despite its widespread use, it has recently been shown that SHAP can provide misleading information about the relative importance of features for classification \cite{shap_failings,shap_refute_additional,shap_refute,shap_scores}.
As discussed in  \cite{trustableAI},
another limitation
of SHAP is that, unlike abductive explanation, it does not take minimality
of an explanation into account. 
To compare SHAP with our method, we fixed a threshold $\delta  $ for the SHAP
score and
then identified  the set of relevant features as those features whose SHAP score is strictly higher than 
$\delta  $
if positive and 
strictly lower than $-\delta  $
  if negative. 
We observed that SHAP 
considered
relevant some input
features having zero weight
connections 
with the network's hidden layer. 
This aspect
is visually represented in Figure \ref{fig:SHAPrelevance}. 
\begin{figure}[htb]
    \begin{center}
    \includegraphics[width=195pt]{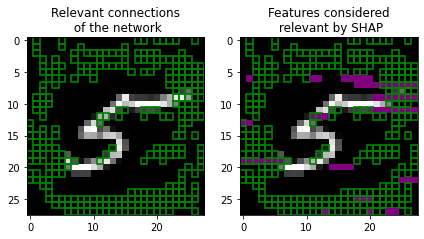} 
    \caption{Green
    features in the two
    figures are those having
    non-zero weight connections
    with the network's  hidden layer.
    Features in purple
    on the right figure
    are considered relevant by SHAP.} \label{fig:SHAPrelevance}
    \end{center}
\end{figure}
This  is a consequence of the model-agnostic nature of ``black box'' explainability methods
like 
SHAP.

Table \ref{table3} summarizes the results on the time required to compute the SHAP score for an input feature, as well as the percentage of features with zero-weight connections to the hidden layer that SHAP incorrectly identified as relevant, across different sample space sizes. On average, 47\% of the input features deemed relevant by SHAP had zero-weight connections to the network's hidden layer. 
As the table  shows, increasing the sample space size reduces the percentage of wrongly considered features, but at the cost of increased computation time.
The performance of SHAP contrasts sharply with what is demonstrated in Proposition \ref{prop:existencelink}: our method guarantees that explanations never include input features with zero-weight connections to the hidden layer.

\section{Conclusion}\label{concl}

Let us take stock. We proposed a causal analysis of Binary Spiking Neural Networks (BSNNs) by mapping their spiking dynamics to binary causal models (BCMs). This mapping enabled the computation of abductive explanations for BSNN decisions in the context of the MNIST classification task, using both SAT-based and SMT-based approaches. Additionally, we compared our logic-based method to SHAP and demonstrated that, unlike SHAP, our approach reliably excludes causally irrelevant features from explanations. 
In the current work, we focused exclusively on the notion of abductive explanation (AXp).
Future research will aim to extend our causal analysis of BSNNs to encompass more sophisticated concepts, including actual cause \cite{HalpernPearl2005a} and NESS (Necessary Element of a Sufficient Set) cause \cite{DBLP:conf/aaai/Beckers21a,DBLP:conf/kr/Halpern08a}. 
Our causal framework provides the necessary expressiveness to formally capture these concepts, and we believe that the logic-based approach we employed for computing abductive explanations can be extended to compute these  notions as well. Another promising direction for future research is to develop a causal analysis of convolutional BSNNs (C-BSNNs) \cite{srinivasan2019restocnet} following the approach presented here. We anticipate that incorporating convolutional layers could enhance accuracy on more complex datasets. Finally, we plan to extend our logic-based causal framework beyond simple visual classification tasks, applying it to explain BSNNs trained on language datasets \cite{Dbal_sengupta}.

\Appendix

In this appendix, we present 
i) 
the proofs of the mathematical
results
presented in the paper (Section \ref{bcm:proofs}),
and
ii) 
the SMT encoding of the causal
model for a BSNN
with Boolean weights
and for a BSNN
with ternary weights
(Section \ref{appendix:smt}).

\section{Proofs} \label{bcm:proofs}

\subsection{Proof
of Theorem \ref{theo:character} 
}\label{bcm:proofs1}

\begin{proof}
($\Rightarrow$) We first prove the left-to-right direction.
Suppose
i) 
$ 
 (\mathcal{F}_X)_{X \in \mathbf{N}  } $
 is 
$S$-compatible up to time $\mathsf{t}_{\mathit{end}}$
 and ii) $
 \forall X \in \mathbf{N}, 
 \forall t \leq  \mathsf{t}_{\mathit{end}}, 
 \mathcal{F}_X(t)= 
  \mathcal{I }(p_{X,t })$.
We are going to  prove that
   $ \mathcal{I }
 \models \mathcal{E}(p_{X,t})$
   for every 
  $t\in \{0,\ldots, \mathsf{t}_{\mathit{end}}\} $
  and for every $X\in \mathbf{L}$.
  The  case $t=0$
is evident. In fact,
$\mathcal{I }(p_{X,0} )= \mathcal{F}_X(0)=0$
by i) and ii). 
Moreover, $\mathcal{I }(p_{X,0} )=0$ iff $ \mathit{Val}( \mathcal{I }, p_{X,0}
\leftrightarrow \bot)=1 $,
and 
 $ \mathit{Val}( \mathcal{I }, p_{X,0}
\leftrightarrow \bot)=1 $ iff $\mathcal{I } \models p_{X,0}
\leftrightarrow \bot$. Thus, 
$\mathcal{I } \models p_{X,0}
\leftrightarrow \bot$
which is equivalent to
$\mathcal{I } \models \mathcal{E}(p_{X,0}) $. 
Let us prove the case $t>0$ by reductio ad absurdum.
Suppose, toward a contradiction, that 
$\mathcal{I } \not  \models \mathcal{E}(p_{X,t}) $.
The latter
is equivalent to  $ \mathit{Val}( \mathcal{I },
 \mathcal{E}(p_{X,t}) )=0 $
which is equivalent to iii)
$ \mathit{Val}( \mathcal{I },
p_{X,t} )=0 $
and
$ \mathit{Val}( \mathcal{I }, \chi )=1 $, or iv) 
$ \mathit{Val}( \mathcal{I },
p_{X,t} )=1 $
and
$ \mathit{Val}( \mathcal{I }, \chi )=0 $,
where $\chi$
abbreviates the following Boolean expression:
\begin{align*}
   \chi=_{\mathit{def}}    \Big( \neg p_{ X , t-1} \rightarrow 
    \bigvee\limits_{ \substack{ \Omega \subseteq \mathcal{R}^+(X):   \\  
      \mathcal{A}(X,t-1)  
    +|\Omega | \geq   \tau_X }} \big(  \bigwedge\limits_{X' \in \Omega } p_{X', t } \big) \Big) \wedge \\   \Big(   p_{ X , t-1} \rightarrow 
    \bigvee\limits_{ \substack{ \Omega  \subseteq \mathcal{R}^+(X) : 
     \\     |\Omega | \geq   \tau_X }} \big(\bigwedge\limits_{X' \in \Omega } p_{X', t }\big) \Big) . 
\end{align*}
Suppose iii) holds. On the one hand,
we 
have 
$ \mathit{Val}( \mathcal{I },
p_{X,t} )=0 $
iff 
$\mathcal{I }(p_{X,t} )=0$,
and, by i) and  ii), we have 
$\mathcal{I }(p_{X,t} )=0$
iff 
$\mathcal{F}_X(t)=\Theta\big(    
    \mathcal{A}(X,t) - \tau_X 
       \big) =0$. 
       Hence, by iii), we have 
       $ \Theta\big(    
    \mathcal{A}(X,t) - \tau_X 
       \big) =0$.
       On the other hand, by ii), 
       it is routine mathematical exercise
       to verify that 
    $\mathit{Val}(\mathcal{I }, \chi)=  \Theta\big(    
    \mathcal{A}(X,t) - \tau_X 
       \big)$. 
       Hence, by iii), we have that 
       $\Theta\big(    
    \mathcal{A}(X,t) - \tau_X 
       \big)= 1$
       which leads to a contradiction. 
   In an analogous way we can prove 
   that iv) leads to a contradiction.

($\Leftarrow$) We are going to prove the right-to-left direction.
Suppose
i) 
 $ \mathcal{I }
 \models \bigwedge_{ p_{X,t} \in  \mathbf{V}_S}
 \mathcal{E}_S(p_{X,t})$
 and ii) $
 \forall X \in \mathbf{N}, 
 \forall t \leq  \mathsf{t}_{\mathit{end}}, 
 \mathcal{F}_X(t)= 
  \mathcal{I }(p_{X,t })$.
We are going to  prove that
  $ (\mathcal{F}_X)_{X \in \mathbf{N}  } $
 is  $S$-compatible up to time $\mathsf{t}_{\mathit{end}}$,
 that is, 
$\mathcal{F}_X(0)=0$
and 
$\mathcal{F}_X(t)=\Theta\big(    
    \mathcal{A}(X,t) - \tau_X 
       \big)$
for every $0 <t \leq \mathsf{t}_{\mathit{end}}$. 
  The  case $t=0$
is evident. 
In fact, $\mathcal{I }(p_{X,0} )=0$ iff $ \mathit{Val}( \mathcal{I }, p_{X,0}
\leftrightarrow \bot)=1 $,
and 
 $ \mathit{Val}( \mathcal{I }, p_{X,0}
\leftrightarrow \bot)=1 $ iff $\mathcal{I } \models p_{X,0}
\leftrightarrow \bot$. 
Thus, 
$\mathcal{I }(p_{X,0} )= \mathcal{F}_X(0)=0$
by i) and ii). 
Let us prove the case $0 <t \leq \mathsf{t}_{\mathit{end}}$ by reductio ad absurdum.
Suppose, toward a contradiction, that 
$\mathcal{F}_X(t) \neq  \Theta\big(    
    \mathcal{A}(X,t) - \tau_X 
       \big)$. 
By i), we have
$\mathcal{I }  \models \mathcal{E}_S(p_{X,t}) $.
The latter
is equivalent to  $ \mathit{Val}( \mathcal{I },
 \mathcal{E}_S(p_{X,t}) )=1 $
which is equivalent to iii)
$ \mathit{Val}( \mathcal{I },
p_{X,t} )=1 $
and
$ \mathit{Val}( \mathcal{I }, \chi )=1 $, or iv) 
$ \mathit{Val}( \mathcal{I },
p_{X,t} )=0 $
and
$ \mathit{Val}( \mathcal{I }, \chi )=0 $,
where $\chi$
is the same abbreviation
as in the proof
of the $\Rightarrow$-direction.
Suppose iii) holds. On the one hand,
we 
have 
$ \mathit{Val}( \mathcal{I },
p_{X,t} )=1 $
 iff 
$\mathcal{I }(p_{X,t} )=1 $,
and, by   ii), we have 
$\mathcal{I }(p_{X,t} )=
\mathcal{F}_X(t)  $. 
       Hence, by iii), we have 
       $ \mathcal{F}_X(t) =1$.
       On the other hand, by ii), 
       it is routine mathematical exercise
       to verify that 
    $\mathit{Val}(\mathcal{I }, \chi)=  \Theta\big(    
    \mathcal{A}(X,t) - \tau_X 
       \big)$. 
       Hence, by iii), we have that 
       $\Theta\big(    
    \mathcal{A}(X,t) - \tau_X 
       \big)= 1$
       and, consequently, 
         $\mathcal{F}_X(t)= 1$.
         This leads to a contradiction. 
   In an analogous way we can prove 
   that iv) leads to a contradiction.
\end{proof}

\subsection{Proof
of Proposition \ref{prop:existencelink}  
}\label{bcm:proofs2}

\begin{proof}
    Suppose i) 
  the term   $\lambda= p_{\mathfrak{I}_{x,y }, t }  \wedge  \lambda'$
    is  an abductive explanation of 
$ \mathsf{out}_{S_{k}^{\mathit{bin} },t } $ and, toward a contradiction, 
  ii)   $   \not  \exists \mathfrak{H}_{z } \in \mathbf{H}^k  \text{ such that }
   \mathfrak{I}_{x,y } \in  \mathcal{R}^+ ( \mathfrak{H}_{z })$.
By ii), we have that iii) for every $p_{X,t' }\in \mathbf{V}_{S_{k}^{\mathit{bin} }} $
   the Boolean equation $\mathcal{E}(p_{X,t'  }) $
   does not contain the variable $p_{\mathfrak{I}_{x,y }, t }$.    
Moreover,
by the definition of a term 
and since $ p_{\mathfrak{I}_{x,y }, t } \in \mathbf{U}_{S_{k}^{\mathit{bin} }}$, 
iv) 
$p_{\mathfrak{I}_{x,y }, t }$ does not appear in $\lambda' $
and 
$p_{\mathfrak{I}_{x,y }, t }$ does not appear in $ \mathsf{out}_{S_{k}^{\mathit{bin} },t }  $. 
By iii) and iv), 
we have that v) 
$    \models
    \big(\bigwedge_{p_{X,t' } \in \mathbf{V}_{S_{k}^{\mathit{bin} }}  }
    \mathcal{E}(p_{X,t' }) \wedge 
   p_{\mathfrak{I}_{x,y }, t }  \wedge  \lambda' \big)  \rightarrow \mathsf{out}_{S_{k}^{\mathit{bin} },t } $ iff
   $    \models
    \big(\bigwedge_{p_{X,t' } \in \mathbf{V}_{S_{k}^{\mathit{bin} }}  }
    \mathcal{E}(p_{X,t' }) \wedge   \lambda' \big)  \rightarrow \mathsf{out}_{S_{k}^{\mathit{bin} },t } $. 
   Item i) implies that $    \models
    \big(\bigwedge_{p_{X,t' } \in \mathbf{V}_{S_{k}^{\mathit{bin} }}  }
    \mathcal{E}(p_{X,t' }) \wedge 
   p_{\mathfrak{I}_{x,y }, t }  \wedge  \lambda' \big)  \rightarrow \mathsf{out}_{S_{k}^{\mathit{bin} },t } $ and $ \not
        \models  \big(\bigwedge_{p_{X,t' } \in \mathbf{V}_{S_{k}^{\mathit{bin} }} }
    \mathcal{E}(p_{X,t' }) \wedge 
    \lambda'  \big) \rightarrow \mathsf{out}_{S_{k}^{\mathit{bin} },t } $,
    which is in  contradiction  with v). 
\end{proof}

\section{SMT Encodings}\label{appendix:smt}

In this section 
we present
the SMT representations
(or encodings) 
of the binary 
causal
model for a BSNN
with Boolean weights
and for a BSNN
with ternary weights.

\subsection{Boolean weights}

Given 
 the architecture of a SNN with Boolean weights
 $S = \big(\mathbf{I}, \mathbf{L}, 
\mathcal{R},\mathcal{W},
\{ 0, 1 \} 
,(\tau_X )_{ X \in  \mathbf{L} }
\big)$, 
the SMT
representation
of the binary causal
model for
$S$
is a triplet 
$\Gamma_S  = \big(\mathbf{U}_S, 
\mathbf{V}_S,   \mathcal{E }_S^{smt}
\big)$
where $\mathbf{U}_S$
and $\mathbf{V}_S$
are, respectively, the 
the sets of exogeneous and endogeneous variables 
in the sense of Definition \ref{def:causalmodel}, 
and $\mathcal{E }_S^{smt}$
is the function mapping each
endogenous variable in 
$\mathbf{V}_S$ to a SMT
expression 
such that, $\forall X \in \mathbf{L} $:
\begin{align*}
    \mathcal{E }_S^{smt}(  
    p_{ X , 0 })  =     p_{ X , 0 }= 0, 
\end{align*}
and     
  for   $t>0$:
\begin{align*}
   \mathcal{E }_S^{smt}(  p_{X,t} )= \Big( 
    p_{ X , t } = 1  \leftrightarrow \big( (p_{ X , t-1 } = 0 \rightarrow  \sum_{X' \in \mathcal{R}^+(X)} p_{ X' , t } \\ + \mathcal{A}(X,t-1)  \geq   \tau_X) \ \wedge \\ (p_{ X , t-1 } = 1 \rightarrow  \sum_{X' \in \mathcal{R}^+(X)} p_{ X' , t } \geq   \tau_X) \big) \Big).
\end{align*}

In order to compute abductive
explanations using the SMT encoding
we use a translation of Definition
\ref{def:axpdef}
based on SAT 
into SMT. 
Specifically, let 
$S$
be a SNN with Boolean weights,
$\Gamma_S  = \big(\mathbf{U}_S, 
\mathbf{V}_S, \mathcal{E }_S
\big)$ its BCM, 
$\mathcal{I }_{\mathbf{U} } :
\mathbf{U}  
\longrightarrow \{0,1\}
$
a 
Boolean
interpretation for
the exogenous   variables in $\mathbf{U}_S$,
$\lambda \in \mathit{Term}_{\mathbf{U}_S} $  
and $\omega_0$
a Boolean expression 
built from $   \mathbf{V}_S$. 
We can check whether  $\lambda$
is an abductive explanation
(AXp) of
$\omega_0$
relative to $\Gamma_S $ and  $\mathcal{I }_{\mathbf{U}_S } $
by checking whether the following three
SMT
conditions 
are satisfied:
\begin{align*} 
 i)  & \ \mathcal{I }_{\mathbf{U}_S }  \models     \mathit{tr}(\lambda)  ,\\
    ii)  &
    \models
    \big(\bigwedge_{p \in \mathbf{V}_S }
      \mathcal{E }_S^{smt}(p) \wedge 
  \mathit{tr}( \lambda)
      \wedge  \mathsf{hyp}_S  
      \big)  \rightarrow    \mathit{tr}(\omega_0),\\
        iii) & \
        \forall \lambda' \subset \lambda, \not
        \models  \big(\bigwedge_{p \in \mathbf{V}_S }
     \mathcal{E }_S^{smt}(p) \wedge 
      \mathit{tr}( \lambda' )
      \wedge  \mathsf{hyp}_S  
      \big) \rightarrow    \mathit{tr}
      (\omega_0) ,
\end{align*}
where
\begin{align*}
    \mathit{tr}(p_{X,t}) &= 
    (p_{X,t}=1),\\
        \mathit{tr}(\neg \omega ) &=
        \neg     \mathit{tr}(\omega),\\
               \mathit{tr}(\omega_1\wedge \omega_2 ) &=
           \mathit{tr}(\omega_1)
           \wedge     \mathit{tr}(\omega_2),
\end{align*}
and
 \begin{align*}
 \mathsf{hyp}_S  
 =_{\mathit{def}}
 \bigwedge_{p_{ X , t } \in
 \mathbf{U}_S \cup
\mathbf{V}_S
 }
 \big( p_{ X , t } = 1 \vee p_{ X , t } = 0 \big) .
\end{align*}

\subsection{Ternary  weights}

For a BSNN with ternary weights
$S = \big(\mathbf{I}, \mathbf{L}, 
\mathcal{R},\mathcal{W},
\{-1, 0, 1 \} 
,(\tau_X )_{ X \in  \mathbf{L} }
\big)$
we just need a 
different SMT representation 
of its binary causal model.
In particular, 
in the case of ternary weights
the function 
$\mathcal{E }_S^{smt}$
 should map each
endogenous variable in 
$\mathbf{V}_S$ to the following SMT
expressions, $\forall X \in \mathbf{L} $: 
\begin{align*}
    \mathcal{E }_S^{smt}(  
    p_{ X , 0 })  =     p_{ X , 0 }= 0, 
\end{align*}
and     
  for   $t>0$:
\begin{align*}
    \mathcal{E }_S^{smt}(  p_{X,t} )=  \Big( p_{ X , t } = 1  \leftrightarrow \big( (p_{ X , t-1 } = 0 \rightarrow  \sum_{X' \in \mathcal{R}^+(X)} p_{ X' , t } \\ - \sum_{X'' \in \mathcal{R}^-(X)} p_{ X'' , t }  + \mathcal{A}(X,t-1)  \geq   \tau_X) \ \wedge \\ (p_{ X , t-1 } = 1 \rightarrow  \sum_{X' \in \mathcal{R}^+(X)} p_{ X' , t }-\sum_{X'' \in \mathcal{R}^-(X)} p_{ X'' , t }  \geq   \tau_X) \big) \Big).
\end{align*}








\bibliographystyle{LNGAI}
\bibliography{LNGAI}

@String{Computing = "Computing" }

@String{Computer = "{IEEE} Computer" }

@String{Springer = "Springer-Verlag" }

@inproceedings{DBLP:conf/ijcai/AudemardBLM23,
  author       = {G. Audemard and
                  S. Bellart and
                  J.{-}M. Lagniez and
                  P. Marquis},
  title        = {Computing Abductive Explanations for Boosted Regression Trees},
  booktitle    = {Proceedings of the Thirty-Second International Joint Conference on
                  Artificial Intelligence (IJCAI 2023)},
  pages        = {3432--3441},
  publisher    = {ijcai.org},
  year         = {2023}
}

@inproceedings{DBLP:conf/kr/AudemardKM20,
  author       = {G. Audemard and
                  F. Koriche and
                  P. Marquis},
  title        = {On Tractable {XAI} Queries based on Compiled Representations},
  booktitle    = {Proceedings of the 17th International Conference on Principles of
                  Knowledge Representation and Reasoning (KR 2020)},
  pages        = {838--849},
  year         = {2020}
}

@inproceedings{ijcai2024p52,
  title     = {Relevant Irrelevance: Generating Alterfactual Explanations for Image Classifiers},
  author    = {Mertes, S. and Huber, T. and Karle, C. and Weitz, K. and Schlagowski, R. and Conati, C. and André, E. },
  booktitle = {Proceedings of the Thirty-Third International Joint Conference on
               Artificial Intelligence {IJCAI 2024}},
  publisher = {ijcai.org },
  pages     = {467--475},
  year      = {2024}
}

@inproceedings{DBLP:journals/corr/abs-2008-05803,
  author       = {J. Marques-Silva and
                  T. Gerspacher and
                  M. C. Cooper and
                  A. Ignatiev and
                  N. Narodytska},
  title        = {Explaining Naive Bayes and Other Linear Classifiers with Polynomial Time and Delay},
  booktitle    = {Advances in Neural Information Processing Systems 33: Annual Conference on Neural Information Processing Systems 2020},
  year         = {2020}
}

@inproceedings{DBLP:conf/ijcai/Izza021,
  author       = {Y. Izza and
                  J. Marques-Silva},
  title        = {On Explaining Random Forests with {SAT}},
  booktitle    = {Proceedings of the Thirtieth International Joint Conference on Artificial
                  Intelligence (IJCAI 2021)},
  pages        = {2584--2591},
  publisher    = {},
  year         = {2021} 
}

@incollection{BSST09,
   author = {C. Barrett and R. Sebastiani and S. Seshia and
	C. Tinelli},

   title = {Satisfiability Modulo Theories},
   booktitle = {Handbook of Satisfiability},
   series = {Frontiers in Artificial Intelligence and Applications},
   volume = {185},
   chapter = {26},
   pages = {825--885},
   publisher = {IOS Press},
   year = {2009}
}

@article{DBLP:journals/logcom/LiuL23,
  author       = {X. Liu and
                  E. Lorini},
  title        = {A unified logical framework for explanations in classifier systems},
  journal      = {Journal of  Logic and  Computation },
  volume       = {33},
  number       = {2},
  pages        = {485--515},
  year         = {2023}
}

@article{HalpernPearl2005a,
  title={Causes and explanations:
a structural-model approach. {P}art {I}: Causes},
  author={Halpern, J. Y. AND  Pearl, J.  },
  journal={British
Journal for Philosophy of Science},
  volume={56},
  number={4},
  pages={843--887},
  year={2005}
}

@inproceedings{DBLP:conf/kr/Halpern08a,
  author    = {J. Y. Halpern},
  title     = {Defaults and Normality in Causal Structures},
  booktitle = {Principles of Knowledge Representation and Reasoning: Proceedings
               of the Eleventh International Conference (KR 2008) },
  pages     = {198--208},
  publisher = {{AAAI} Press},
  year      = {2008},
  url       = {},
  bibsource = {}
}

@inproceedings{DBLP:conf/aaai/Beckers21a,
  author    = {S. Beckers},
  title     = {The Counterfactual {NESS} Definition of Causation},
  booktitle = {Proceedings of the Thirty-Fifth {AAAI} Conference on Artificial Intelligence (AAAI-21) },
  pages     = {6210--6217},
  publisher = {{AAAI} Press},
  year      = {2021}
}

@inproceedings{DBLP:conf/aaai/Potyka21,
  author       = {N. Potyka},
  title        = {Interpreting Neural Networks as Quantitative Argumentation Frameworks},
  booktitle    = {Proceedings
      of the Thirty-Fifth {AAAI} Conference on Artificial Intelligence (AAA-21) },
  pages        = {6463--6470},
  publisher    = {{AAAI} Press},
  year         = {2021}
}

@inproceedings{DBLP:conf/ecai/AyoobiPT23,
  author       = {H. Ayoobi and
                  N. Potyka and
                  F. Toni},
  title        = {SpArX: Sparse Argumentative Explanations for Neural Networks},
  booktitle    = {
Proceedings of the 26th
       European Conference
       on Artificial
       Intelligence (ECAI 2023)},
  series       = {Frontiers in Artificial Intelligence and Applications},
  volume       = {372},
  pages        = {149--156},
  publisher    = {{IOS} Press},
  year         = {2023}
}

@InProceedings{NIPS2016_d8330f85,
 author = {I. Hubara  and M. Courbariaux  and D. Soudry  and R. El-Yaniv  and Y. Bengio },
 booktitle = {Advances in Neural Information Processing Systems},
 editor = {D. Lee and M. Sugiyama and U. Luxburg and I. Guyon and R. Garnett},
 pages = {4107--4115},
 publisher = {Curran Associates, Inc.},
 title = {Binarized Neural Networks},
 year = {2016} 
}

@InProceedings{10.1007/978-3-319-46493-0_32,
author="M. Rastegari
and V. Ordonez
and J. Redmon
and A. Farhadi",
editor="Leibe, Bastian
and Matas, Jiri
and Sebe, Nicu
and Welling, Max",
title="XNOR-Net: ImageNet Classification Using Binary Convolutional Neural Networks",
booktitle="Computer Vision -- ECCV 2016",
year="2016",
publisher="Springer International Publishing",
address="Cham",
pages="525--542",
isbn=""
}

@article{bengio2013estimating,
  author       = { Y. Bengio and
                   N. L{\'{e}}onard and
                   A. Courville},
  title        = {Estimating or Propagating Gradients Through Stochastic Neurons for
                  Conditional Computation},
  journal      = {CoRR},
  volume       = {abs/1308.3432},
  year         = {2013},
  url          = {},
  eprinttype    = {arXiv},
  eprint       = {1308.3432},
  bibsource    = {}
}

@inproceedings{Tang_Hua_Wang_2017,
  author       = {W. Tang and
                  G. Hua and
                  L. Wang},
  title        = {How to Train a Compact Binary Neural Network with High Accuracy?},
  booktitle    = {Proceedings of the Thirty-First {AAAI} Conference on Artificial Intelligence (AAAI-17)},
  pages        = {2625--2631},
  publisher    = {{AAAI} Press},
  year         = {2017} 
}

@article{QIN2020107281,
title = {Binary neural networks: A survey},
journal = {Pattern Recognition},
volume = {105},
pages = {107281},
year = {2020},
issn = {0031-3203},
doi = {},
url = {},
author = { H. Qin and  R. Gong and  X. Liu and X. Bai and J. Song and N. Sebe}
}

@inproceedings{DBLP:conf/ecai/DarwicheH20,
  author    = {A. Darwiche and A. Hirth },
  title     = {On the Reasons Behind Decisions},
  booktitle = {Proceedings of the 24th European Conference on Artificial Intelligence ({ECAI} 2020) },
  pages     = {712--720},
  publisher = {{IOS} Press},
  year      = {2020},
  series = {}
}

@inproceedings{ignatiev2019abduction,
  title={Abduction-based explanations for machine learning models},
  author={A. Ignatiev and N. Narodytska and J. Marques-Silva },
  booktitle={Proceedings of the Thirty-third AAAI Conference on Artificial Intelligence (AAAI-19)},
  pages={1511--1519},
  year={2019}
}

@inproceedings{DBLP:conf/kr/ShiSDC20,
  author       = {W. Shi and
                  A. Shih and
                  A. Darwiche and
                  A. Choi},
  title        = {On Tractable Representations of Binary Neural Networks},
  booktitle    = {Proceedings of the 17th International Conference on Principles of
                  Knowledge Representation and Reasoning (KR 2020)},
  pages        = {882--892},
  year         = {2020}
}

@article{DBLP:journals/ai/CooperS23,
  author       = {M. C. Cooper and
                  J. Marques-Silva},
  title        = {Tractability of explaining classifier decisions},
  journal      = {Artificial  Intelligence },
  volume       = {316},
  pages        = {103841},
  year         = {2023}
}

@article{DBLP:journals/ker/Miller21,
  author       = {T. Miller},
  title        = {Contrastive explanation: a structural-model approach},
  journal      = {The Knowledge Engineering Review},
    volume       = {36},
  year         = {2021} 
}

@article{DBLP:journals/corr/abs-2010-10596,
  author       = {S. Verma and
                  J. P. Dickerson and
                  K. Hines},
  title        = {Counterfactual Explanations for Machine Learning: {A} Review},
  journal      = {CoRR},
  volume       = {},
  year         = {2020},
  url          = {},
  eprinttype    = {arXiv},
  eprint       = {2010.10596}
}

@inproceedings{DBLP:conf/fair/Marquis91,
  author       = {P. Marquis},
  title        = {Extending abduction from propositional to first-order logic},
  booktitle    = {Procedings 
                   of the  International Workshop on 
                   Fundamentals of Artificial Intelligence Research
                 (FAIR'91) },
  series       = {LNCS },
  volume       = {},
  pages        = {141--155},
  publisher    = {Springer},
  year         = {1991}
}

@inproceedings{DBLP:conf/ijcai/ShihCD18,
  author       = {A. Shih and
                  A. Choi and
                  A. Darwiche},
  title        = {A Symbolic Approach to Explaining Bayesian Network Classifiers},
  booktitle    = {Proceedings of the Twenty-Seventh International Joint Conference on
                  Artificial Intelligence (IJCAI 2018)},
  pages        = {5103--5111},
  publisher    = {},
  year         = {2018},
  url          = {},
  doi          = {},
  bibsource    = {}
}

@article{DBLP:journals/jair/ChocklerH04,
  author       = {H. Chockler and
                  J. Y. Halpern},
  title        = {Responsibility and Blame: {A} Structural-Model Approach},
  journal      = {Journal of Artificial Intelligence Research },
  volume       = {22},
  pages        = {93--115},
  year         = {2004},
  url          = {},
  doi          = {},
  bibsource    = {}
}

@article{DBLP:journals/jair/AleksandrowiczC17,
  author       = {G. Aleksandrowicz and
                  H. Chockler and
                  J. Y. Halpern and
                  A.  Ivrii},
  title        = {The Computational Complexity of Structure-Based Causality},
  journal      = {Journal of Artificial Intelligence Research },
  volume       = {58},
  pages        = {431--451},
  year         = {2017},
  url          = {},
  doi          = {},
  bibsource    = {}
}

@inproceedings{DBLP:conf/ijcai/Lorini24,
  author       = {T. de Lima and E. Lorini},
  title        = {Model Checking Causality },
  booktitle    = {Proceedings of the Thirty-Third International Joint Conference on
                  Artificial Intelligence (IJCAI 2024)},
  publisher    = {ijcai.org},
  year         = {2024}
}

@book{Pearl2009,
  author =	 { J. Pearl},
  title =	 {Causality: Models, Reasoning and Inference},
  publisher =	 {Cambridge University Press},
  year =	 2009
}

@book{HalpernBook2016,
  title={Actual causality},
  author={J. Y. Halpern },
  year={2016},
  publisher={MIT Press}
}

@article{Halpern2000,
  title={Axiomatizing causal reasoning},
  author={J. Y. Halpern},
  journal={Journal of Artificial Intelligence Research },
  volume={12},
  pages={317--337},
  year={2000},
  publisher={JSTOR}
}

@article{Lu_2020,
   title={Exploring the Connection Between Binary and Spiking Neural Networks},
   volume={14},
   ISSN={1662-453X},
   url={},
   DOI={},
   journal={Frontiers in Neuroscience},
   publisher={Frontiers Media SA},
   author={S. Lu and A. Sengupta},
   year={2020},
   month=jun }

@article{bs4nn,
author = {S. R. Kheradpisheh and M. Mirsadeghi and T. Masquelier},
year = {2022},
month = {04},
pages = {},
title = {BS4NN: Binarized Spiking Neural Networks with Temporal Coding and Learning},
volume = {54},
journal = {Neural Processing Letters},
doi = {}
}

@article{neftci2019surrogate,
  title={Surrogate Gradient Learning in Spiking Neural Networks: Bringing the Power of Gradient-based optimization to spiking neural networks},
  author={E. O. Neftci and H. Mostafa and F. Zenke},
  journal={IEEE Signal Processing Magazine},
  year={2019},
  volume={36},
  pages={51-63},
  url={}
}

@ARTICLE{srinivasan2019restocnet,

AUTHOR={G. Srinivasan  and K. Roy },

TITLE={ReStoCNet: Residual Stochastic Binary Convolutional Spiking Neural Network for Memory-Efficient Neuromorphic Computing},

JOURNAL={Frontiers in Neuroscience},

VOLUME={13},

YEAR={2019},

URL={},

DOI={},

ISSN={1662-453X}}

@article{Jang2020BiSNNTS,
  title={BiSNN: Training Spiking Neural Networks with Binary Weights via Bayesian Learning},
  author={ H. Jang and N. Skatchkovsky and O. Simeone},
  journal={2021 IEEE Data Science and Learning Workshop (DSLW)},
  year={2020},
  pages={1-6},
  url={}
}

@inproceedings{ijcai2023p366,
  title     = {A Rule-Based Modal View of Causal Reasoning},
  author    = {E. Lorini},
  booktitle = {Proceedings of the Thirty-Second International Joint Conference on
               Artificial Intelligence (IJCAI 2023)},
  pages     = {3286--3295},
  year      = {2023}
}

@article{shap_failings,
  author       = {X. Huang and
                  J. Marques-Silva},
  title        = {On the failings of {S}hapley values for explainability},
  journal      = {International Journal
                  of Approximate Reasoning },
  volume       = {171},
  pages        = {109112},
  year         = {2024},
  url          = {},
  doi          = {},
  bibsource    = {dblp computer science bibliography, https://dblp.org}
}

@article{shap_scores,
  author       = {O. Letoffe and
                  X. Huang and
                  N. Asher and
                  J. Marques-Silva},
  title        = {From SHAP Scores to Feature Importance Scores},
  journal      = {CoRR},
  volume       = {},
  year         = {2024},
  url          = {},
  doi          = {},
  eprinttype    = {arXiv},
  eprint       = {2405.11766},
  bibsource    = {dblp computer science bibliography, https://dblp.org}
}

@article{shap_refute,
  author       = {X. Huang and
                  J. Marques-Silva},
  title        = {A Refutation of {S}hapley Values for Explainability},
  journal      = {CoRR},
  volume       = {abs/2309.03041},
  year         = {2023},
  url          = {},
  doi          = {},
  eprinttype    = {arXiv},
  eprint       = {2309.03041},
  bibsource    = {dblp computer science bibliography, https://dblp.org}
}

@article{shap_refute_additional,
  author       = {X. Huang and
                  J. Marques-Silva},
  title        = {Refutation of {S}hapley Values for {XAI} - Additional Evidence},
  journal      = {CoRR},
  volume       = {abs/2310.00416},
  year         = {2023},
  url          = {},
  doi          = {},
  eprinttype    = {arXiv},
  eprint       = {2310.00416},
  bibsource    = {dblp computer science bibliography, https://dblp.org}
}

@inproceedings{trustableAI,
  title     = {Towards Trustable Explainable {AI}},
  author    = {Ignatiev, A.},
  booktitle = {Proceedings of the Twenty-Ninth International Joint Conference on
               Artificial Intelligence, {IJCAI-20}},
  publisher = {ijcai.org},
  pages     = {5154--5158},
  year      = {2020},
  doi       = {},
  url       = {},
}

@inproceedings{nips_shap,
 author    = {Lundberg, S. M. and Lee, S. I.},
 booktitle = {Advances in Neural Information Processing Systems},
 editor    = {I. Guyon and U. Von Luxburg and S. Bengio and H. Wallach and R. Fergus and S. Vishwanathan and R. Garnett},
 publisher = {Curran Associates, Inc.},
 title     = {A Unified Approach to Interpreting Model Predictions},
 year      = {2017},

}

@article {rate_coding,
	author = {Prescott, S. A. and Sejnowski, T. J.},
	title = {Spike-Rate Coding and Spike-Time Coding Are Affected Oppositely by Different Adaptation Mechanisms},
	volume = {28},
	number = {50},
	pages = {13649--13661},
	year = {2008},
	doi = {},
	publisher = {Society for Neuroscience},
	issn = {},
	URL = {},
	eprint = {},
	journal = {Journal of Neuroscience}
}

@inproceedings{Dbal_sengupta,
  author       = {M. Bal and
                  A. Sengupta},
  title        = {SpikingBERT: Distilling {BERT} to Train Spiking Language Models Using
                  Implicit Differentiation},
  booktitle    = {Proceedings
      of the 
      Thirty-Eighth {AAAI} Conference on Artificial Intelligence (AAAI-24)  },
  pages        = {10998--11006},
  publisher    = {{AAAI} Press},
  year         = {2024}
}

\end{document}